\documentclass[sigconf]{acmart}
\usepackage{booktabs} 
\usepackage[utf8]{inputenc} 
\usepackage[T1]{fontenc}    
\usepackage{hyperref}       
\usepackage{url}            
\usepackage{booktabs}       
\usepackage{amsfonts}       
\usepackage{nicefrac}       
\usepackage{microtype}      
\usepackage{comment}
\usepackage{subfigure} 
\usepackage{amsmath}
\usepackage[ruled,vlined, linesnumbered]{algorithm2e}
\usepackage{graphicx}
\usepackage{enumitem}

\usepackage[colorinlistoftodos]{todonotes}

\newcommand{\argmin}{\operatornamewithlimits{argmin}}

\newcommand{\bx}{\boldsymbol{x}}

\newcommand{\bw}{{\boldsymbol{w}}}

\newcommand{\bc}{{\boldsymbol{c}}}

\newcommand{\bu}{{\boldsymbol{u}}}

\newcommand{\R}{{\mathbb{R}}}

\newcommand{\nys}{{{Nystr\"{o}m }}}
\newcommand{\Nys}{{{Nystr\"{o}m }}}

\usepackage{amsthm}
\newtheorem{theorem}{Theorem}[section]

\setcopyright{none}

\begin{document}
\title{Nonlinear Online Learning with Adaptive \Nys Approximation}

\author{Si Si}
\affiliation{
  \institution{Google Research}
    \city{Mountain View} 
  \state{CA} 
  \postcode{94043}
  }
\email{sisidaisy@google.com}

\author{Sanjiv Kumar}
\affiliation{
  \institution{Google Research}
    \city{New York} 
  \state{NY} 
  \postcode{94043}
  }
\email{sanjivk@google.com}

\author{Yang Li}
\affiliation{
  \institution{Google Research}
    \city{Mountain View} 
  \state{CA} 
  \postcode{94043}
  }
\email{liyang@google.com}

\begin{abstract}
Use of nonlinear feature maps via kernel approximation has led to success in many online learning tasks. As a popular kernel approximation method, \nys approximation, has been well investigated, and various landmark points selection methods have been proposed to improve the approximation quality. However, these improved \nys methods cannot be directly applied to the online learning setting as they need to access the entire dataset to learn the landmark points, while we need to update model on-the-fly in the online setting. To address this challenge, we propose Adaptive \nys approximation for solving nonlinear online learning problems. The key idea is to adaptively modify the landmark points via online kmeans and adjust the model accordingly via solving least square problem followed by a gradient descent step. We show that the resulting algorithm outperforms state-of-the-art online learning methods under the same budget.
\end{abstract}
\keywords{Online Learning, Kernel Methods, Classification.}
\maketitle

\section{Introduction}
Online learning~\cite{Crammer:2006} has been widely used in many machine learning problems, e.g, email filtering and ads advertisement, where data points are observed sequentially. In the online learning setting, there is often a limited budget/memory allowed for storing only a small set of samples, instead of the entire data, and it updates the model and makes prediction on-the-fly. The online learning setting is significantly in contrast with offline learning, where the learning has the access to the entire data, and is able to go through the data multiple times to train a model, leading to higher accuracy than their online counterparts.

Nonlinear online learning~\cite{Dekel05,Orabona:2008} has drawn a considerable amount of attention as it captures nonlinearity in the data which cannot be effectively modeled in a linear online learning method, and usually achieves better accuracy. One group of nonlinear online learning is based on kernels, that is, kernel-based online learning. An online kernel learning method typically requires a budget in memory for storing a set of data points as support vectors for representing a kernel based model. In an online learning process, this support vectors set can grow dramatically. Therefore it is important to keep the number of support vectors (i.e., the budget size) bounded. Several popular techniques have been proposed to remove, project, and merge support vectors as the online learning process proceeds so as to bound the budget size. In contrast to the traditional work on keeping support vectors in the budget under the online learning setting, it is of increasing interest to construct the kernel-based nonlinear mapping for data points to enable nonlinear online learning.

Kernel approximation is a popular method to approximate the kernel mapping and transform a data point from input feature space to another space via non-linear mapping, and then a linear online learning method is applied in this new space. \nys method~\cite{CW01a,Zhang:2008,RC16} is a popular kernel approximation method to construct non-linear kernel mapping for data points, which has shown superior performance compared to support vectors based nonlinear online learning algorithms \cite{nogd}. From the nonlinear feature perspective, the basic idea of \nys is to construct non-linear mapping based on a set of \emph{landmark} points. Many algorithms have been proposed to select the landmark points optimally for improving the quality of the nonlinear mapping from \nys approximation, e.g., kmeans \nys~\cite{Zhang:2008}, leverage score based \nys~\cite{AG13a}; however, these methods can only be applied in an offline setting, and their landmark points need to be computed based on the entire data.

In this paper, our goal is to improve \nys method under online learning scenario, where only a single pass through the data is allowed for updating the model and making prediction. More specifically, instead of using a fixed set of landmark points, which are often selected at the beginning of an online learning process, we propose to adaptively learn the landmark points according to the online data stream. When the distance between a new data point and its closest kmeans centroid (a landmark point) is larger than a threshold, we update the kmeans centroids (landmark points), and thus change the non-linear mapping. Due to the change of mapping, we adopt a two-stage optimization scheme to update the model based on the new mapping by solving a least squares problem followed by a gradient descent step. We compare our method with classical and state-of-the-art online learning algorithms, and provide extensive empirical results on several
real-world datasets for regression and classification under different losses. The experiments show that our method achieves higher accuracy under the same budget. Theoretically, we show that our model has regret bound of $O(\sqrt{T})$ where $T$ is the number of data points in the data stream.

We summarize related work in Section \ref{sec:related}. In Section \ref{sec:background} we formally introduce the online learning problem and show how to generate kernel mapping via kernel approximation. We propose our model in Section \ref{sec:algorithm}. Empirical results on several real-world datasets are in Section \ref{sec:exp}.
\section{Related Work}
\label{sec:related}
Online learning is widely used in embedded systems, such as mobile applications. Popular online learning algorithms can be categorized into linear and nonlinear methods. Linear online learning~\cite{Crammer:2006} has been well investigated in classification and regression tasks, where it learns a linear model in the input feature space. Because linear methods are often incapable of capturing the nonlinearity in the data, nonlinear online learning, especially kernel-based methods~\cite{OGD}, has drawn much attention in the community. In particular, one line of research directly tries to reduce the size of support vectors to fit the budget constraints. Several strategies have been  proposed to bound the size of the support vectors, including removing redundant support vectors and retaining important ones in the budget, e.g.,~\cite{Dekel05}; or projecting support vectors, e.g., \cite{wang10b,Orabona:2008}; or merging/deleting the existing ones~\cite{Le2016DualSG}. Another set of research on nonlinear online learning is to construct the approximate nonlinear mapping for the data points~\cite{nogd,Le2016DualSG}. Several works have also provided a theoretical analysis for online learning~\cite{Rakhlin2011OnlineLB}.

Another line of research has focused on speeding up kernel machines by explicitly constructing nonlinear maps via kernel approximation. Among different kernel approximation methods, random Fourier feature~\cite{Rahimi:2007} and \nys based method~\cite{CW01a,Drineas05} are two popular and efficient methods. The idea of \nys based method is to generate a low-rank approximation based on the landmark points and the corresponding columns of the kernel matrix. A variety of strategies have been proposed to improve standard \nys, where the landmark points are sampled randomly in the dataset. For example, kmeans \nys~\cite{Zhang:2008} uses kmeans centroids as landmark points; ensemble \nys~\cite{SK09a} combines several landmark points together resulting in a better approximation; ~\cite{AG13a} proposes to use leverage scores for sampling. FastNys~\cite{si2016} constructs pseudo-landmark points for computational benefit. 

However, the previously mentioned improved \nys methods cannot be directly used to construct a non-linear map in an online setting, as they need to access and go through the entire data several times to construct a good approximation. The prediction model is then trained on that approximation. For example, kmeans \nys~\cite{Zhang:2008} needs to perform kmeans on the entire data to obtain centroids and use them as landmark points to train a model. \cite{Pegasos2} needs to search in the entire dataset for the active set of points, and then use them as landmark points in \nys method to train SVM model based on Pegasos algorithm. One way to solve this challenge is to use the first incoming points as landmark points to construct non-linear mapping for online learning~\cite{nogd}. The main difference between our method and \cite{nogd} is that we adaptively update the landmarks and its kernel mapping during the entire online learning process, meanwhile adjusting the model if mapping is updated. 

As another popular and efficient kernel approximation method, random Fourier feature~\cite{Rahimi:2007} based methods also show much success for large-scale machine learning applications. The main idea is to approximate the kernel function based on the Fourier transform. There are many variations of random Fourier feature methods. As an instance, Fastfood~\cite{Le:2013} speeds up the projection operation using Hadamard transform; Doubly SGD~\cite{NIPS2014_5238} uses two unbiased stochastic approximation to the function gradient without the need to store the projection matrix; FOGD~\cite{nogd} applies random Fourier features with the online gradient descent to deal with the streaming data; DualSGD~\cite{Le2016DualSG} uses random Fourier feature to store the information of the removed data points to maintain the budget under the online setting; Reparametered Random Feature(RRF) ~\cite{ijcai2017} learns the distribution parameters in random Fourier feature along with the online learning process.

Most of the above random Fourier feature based methods can be applied under the online learning setting, where the data are coming in sequence. We have compared with some representative random Fourier feature based methods in the experimental section.

\section{Nonlinear Online Learning}
\label{sec:background}
Given $T$ data points, $X_T = \{(\bx_1,y_1), \cdots, (\bx_T,y_T)\}$  where $\bx_t\in \R^d$ and $y_t\in \R$, 
we focus on the following nonlinear regularized risk minimization problem: 
\begin{equation}
\underset{\bw}{\min } \frac{\lambda}{2}\|\bw\|_2^2+\frac{1}{T}\sum\limits_{t=1}^T \ell(y_t,\bw\varphi^T(\bx_t)),
 \label{eq:obj_origin}
\end{equation}
where $\varphi(\cdot)$ is a function that nonlinearly maps data point $\bx_t$ to a 
high-dimensional or even infinite-dimensional space;
$\ell$ is a 
loss function. For example,  
$\ell(y,u) = \max(0, 1-y u)$  is the hinge loss in kernel SVM, $\ell(y,u) = (y-u)^2$ is the squared loss in kernel ridge regression, and $\ell(y, u) = \log(1+\exp(-y u))$ is the logistic loss in kernel logistic regression. 
{\bf In the online setting, data arrives as a sequence and the size of entire data is unknown in advance.} Thus, the regularized loss for the $t$-th iteration is defined as,
\begin{equation}
L_t(\bw) = \frac{\lambda}{2}\|\bw\|_2^2+\ell(y_t,\bw\varphi^T(\bx_t)).
\label{eq:losst}
\end{equation}
The goal of online learning is to learn a sequence of models $\bw_t$ such that the regret is minimized. The regret can be defined as:
\begin{equation}
\sum_{t=1}^T (L_t(\bw_t)-L_t(\bw^*)),
\end{equation}
where $\bw^*$ is the optimal model learned assuming access to all the $T$ samples in one go (also known as offline mode). Kernel methods are widely used in learning offline nonlinear models, however, in an online setting, learning a good nonlinear mapping is a very challenging problem. 

In this work, we propose to approximate the nonlinear kernel map by assuming the kernel matrix $G$ for $X_T$ to be low-rank. Among the  
popular kernel approximation methods, \nys approximation 
has garnered much attention\cite{CW01a,Drineas05,Drineas04fastmonte,SK09a}.  \nys method approximates a kernel matrix by sampling $m\ll T$ landmark points $\{\bu_j\}_{j=1}^m$, 
and forming two matrices $C\in \R^{T\times m}$ and 
$E\in \R^{m\times m}$ based on the kernel function $k(\cdot, \cdot)$, where $C_{ij} = k(\bx_i, \bu_j)$ and $E_{ij}=k(\bu_i, \bu_j)$, and then approximates the kernel matrix $G$ as 
  \begin{equation}
    G\approx \bar{G} := C E^{\dagger} C^T, 
    \label{eq:nystrom_origin}
    \end{equation}
where $E^{\dagger}$ denotes the pseudo-inverse of $E$. From the feature point of view, \Nys method may be viewed as constructing nonlinear features for $\bx_i$. Mathematically, with $m$ landmark points $\bu_1,\cdots, \bu_m$, the nonlinear feature mapping for $\bx_i$ is 
\begin{equation}
 \varphi_M(\bx_i) = [k(\bx_i,\bu_1),k(\bx_i,\bu_2),\cdots,k(\bx_i,\bu_m)]{E^\dagger}^{\frac{1}{2}},
\label{eq:nys_mapping}
\end{equation}
where $M = \{\bu_1,\cdots, \bu_m\}$. There exist a lot of works on how to choose landmark points $\bu_1,\cdots, \bu_m$ to achieve better nonlinear mapping with low approximation error \cite{SK12a,AG13a,Zhang:2008}. 

Given the \nys mapping $\varphi_M(\cdot)$, the loss for $t$-th iteration in Eq~(\ref{eq:losst}) can be written as, 
\begin{equation}
L_t(\bw) = \frac{\lambda}{2}\|\bw\|^2+\ell(y_t, \bw\varphi_M^T(\bx_t)).
\label{eq:nonmap}
\end{equation} 
Note that $\bw$ represents the classification or regression model, and the feature mapping $\varphi_M(\cdot)$
is parameterized by the landmark points $\bu_1, \cdots, \bu_m$. 

\section{Proposed Method}
\label{sec:algorithm}
In this section, we introduce our Online Adaptive \nys Approximation (OANA), a simple way to update the landmark points in \nys approximation along the online learning process. We then propose our main algorithm, Nonlinear Online Learning with Adaptive \Nys Approximation (NOLANA), which applies OANA to construct nonlinear feature mapping, and updates the model accordingly to fit the new data.

\subsection{Online Adaptive \Nys Approximation}
\label{sec:oana}
We start with the description of our online kernel approximation method: OANA. There has been substantial research on improving \nys approximation and constructing kernel mapping, e.g., kmeans \nys\cite{Zhang:2008} and MEKA\cite{meka}, but they need to store the entire data, which is infeasible under the online learning setting. In the previous work~\cite{nogd}, to use \nys approximation and its mapping under the online learning setting, the first $m$ data points in the data stream were fixed as landmark points $M$ and the nonlinear mapping $\varphi_M(\cdot)$ was represented as in Eq~(\ref{eq:nys_mapping}). There is one issue with using a fixed $M$: the initial $m$ samples as landmark points may be a poor choice for learning a good mapping, as they might not represent the characteristic of the data stream. Here, we proposed Online Adaptive \nys Approximation(OANA) to address this issue.

{\bf Online Landmark Points Updates:} The key idea of OANA is to update landmarks using online kmeans. As shown in many previous works e.g.,  \cite{Zhang:2008} and \cite{meka}, using kmeans centroids as landmark points in \nys approximation achieves better approximation than using randomly sampled landmark points from the dataset; however, to achieve good accuracy, kmeans algorithm needs to make multiple passes through data. This is infeasible under the online setting where we have budget constraint and cannot store the entire data. Instead, we adopt the idea of online kmeans, and use its centroids as landmark points to construct an approximate kernel mapping. 

More specifically, at iteration $t$, assume we already have $m$ online kmeans centroids $\bu_1,\cdots,\bu_m$. Then, given a new sample $\bx_t$, its nonlinear kernel mapping is, 
\begin{equation}
\varphi_M(\bx_t) = [k(\bx_t,\bu_1),k(\bx_t,\bu_2),\cdots,k(\bx_t,\bu_m)]U_rS_r^{-\frac{1}{2}}
\label{eq:mapping2}
\end{equation}
where $E_{ij} = k(\bu_i,\bu_j)$ and its rank-$r$ SVD is $E \approx U_rS_rU_r^T$.

Next, we update the centroids. To make the computation more efficient, we only update the kmeans centroids if the distance between $\bx_t$ to its closest cluster $q$ is larger than $\epsilon$, i.e., if $\|\bx_t-\bu_q\|^2\geq \epsilon$, then we will update the $q$-th centroid $\bu_q$ as 
\begin{equation}
\bu_q \leftarrow \frac{N_q\bu_q+\bx_t}{N_q+1},
\end{equation}
where $N_q$ is the number of points at iteration $t$ in the $q$-th cluster, which is updated continuously during the online learning process. If $\|\bx_t-\bu_q\|^2 < \epsilon$ , we keep the centroids or landmark points unchanged and thus the nonlinear mapping does not change as well. In Table \ref{tab:eps}, we will show that using this strategy for online learning leads to less frequent landmark points updates, and improves the time complexity for update as well as the prediction accuracy.

{\bf Fast SVD Update:} If we change $q$-th centroid or landmark point, we need to recompute the SVD of the small $m\times m$ kernel matrix $\bar{E}$ with $\bar{E}_{ij} = k(\bu_i,\bu_j)$. Note that $\bar{E}$ is a rank-2 update of the previous landmark points' kernel matrix $E$, that is, replacing $q$-th row and $q$-th column in $\bar{E}$ with the kernel value between new $\bar{\bu}_q$ with all other landmark points. That is 
\begin{align}
\bar{E} &= E+ab^T+ba^T,\label{eq:Echange}\\
a &= [0,\cdots, 1, 0,\cdots,0]^T, \nonumber\\
b &= [k(\bu_0,\bar{\bu}_q)-k(\bu_0,\bu_{q}),\cdots,\frac{k(\bar{\bu}_q,\bar{\bu}_q)-k(\bu_q,\bu_q)}{2},\nonumber\\
&\cdots,k(\bu_m,\bar{\bu}_q)-k(\bu_m,\bu_q)]^T, \nonumber
\end{align}
where $a$ and $b$ are two length-$m$ vectors with $a_q=1$ and $b$ contains the kernel values changes between $E$ and $\bar{E}$. 

There are several ways we can update the $\bar{E}$'s SVD efficiently without computing it from scratch based on Eq.(\ref{eq:Echange}), such as computing rank-2 update of SVD \cite{BRAND200620,HOEGAERTS2007220} or using randomized SVD\cite{Tropp11} with warm start. 

Let us explain how to use randomized SVD with warm start to compute $\bar{E}$'s SVD. According to \cite{Tropp11}, the main step to approximate the SVD of $\bar{E}$ using randomized SVD is to compute the following equation for $p$ times:
\begin{equation}
Q\leftarrow \bar{E}Q,
\label{eq:randsvd}
\end{equation}
where $Q$ is a random matrix. Empirically, we observe that the eigenspaces of $E$ and $\bar{E}$ are very similar as these two matrices are only different in one row and one column. Motivated by \cite{SiSDP14}, as we know the approximate SVD of $E\approx USU^T$, where $U$ consists of top $r$ eigenvectors and $S$ contains their corresponding eigenvalues, and $E$ and $\bar{E}$ are similar, so we use $Q=U$ 
to warm start Eq.(\ref{eq:randsvd}). It only takes 2 or 3 iterations of Eq.(\ref{eq:randsvd}) ($p=2$ or $3$) to converge to a good approximation of $\bar{E}$'s SVD, which can speed up the SVD computation. The time complexity is $O(m^2r)$. In the budget we only store $E$'s SVD, so we replace $\bar{E}$ with $USU^T+ab^T+ba^T$ when computing Eq.(\ref{eq:randsvd}).

The second method to compute $\bar{E}$'s SVD is rank-2 update SVD to directly use the fact that $\bar{E}=E+ab^T+ba^T$. Updating SVD for rank one matrix perturbation has been well studied\cite{BRAND200620,HOEGAERTS2007220}. Although $\bar{E}$ is a rank-2 update of $E$, we can transform rank-2 problem into two rank-1 SVD update problems.

The comparison between computing SVD from scratch, rank-2 update SVD, and randomized SVD with warm start is shown in Figure \ref{fig:svdtime}.  Here we show the entire processing time (including updating and prediction time, not just the time for using different SVD solvers), and the only difference among these three curves is on the SVD solver. We can see randomized SVD with warm start is the fastest. The reason for rank-2 SVD update SVD to be slow is because in the regime we are working on, $m$ (the size of $K$) is not large, usually to be a few hundreds, and $r$ is similar to $m$ ($r=0.8m$), and we need to use rank-1 SVD update twice to get rank-2 SVD update, therefore there is no much gain for using rank-2 update SVD in this case. Also, the prediction accuracies are similar with different SVD solvers. Based on the above analysis, we use randomized SVD with warm start in our algorithm.

\begin{figure}[htp]
\centering
	{\includegraphics[width=.4\textwidth]{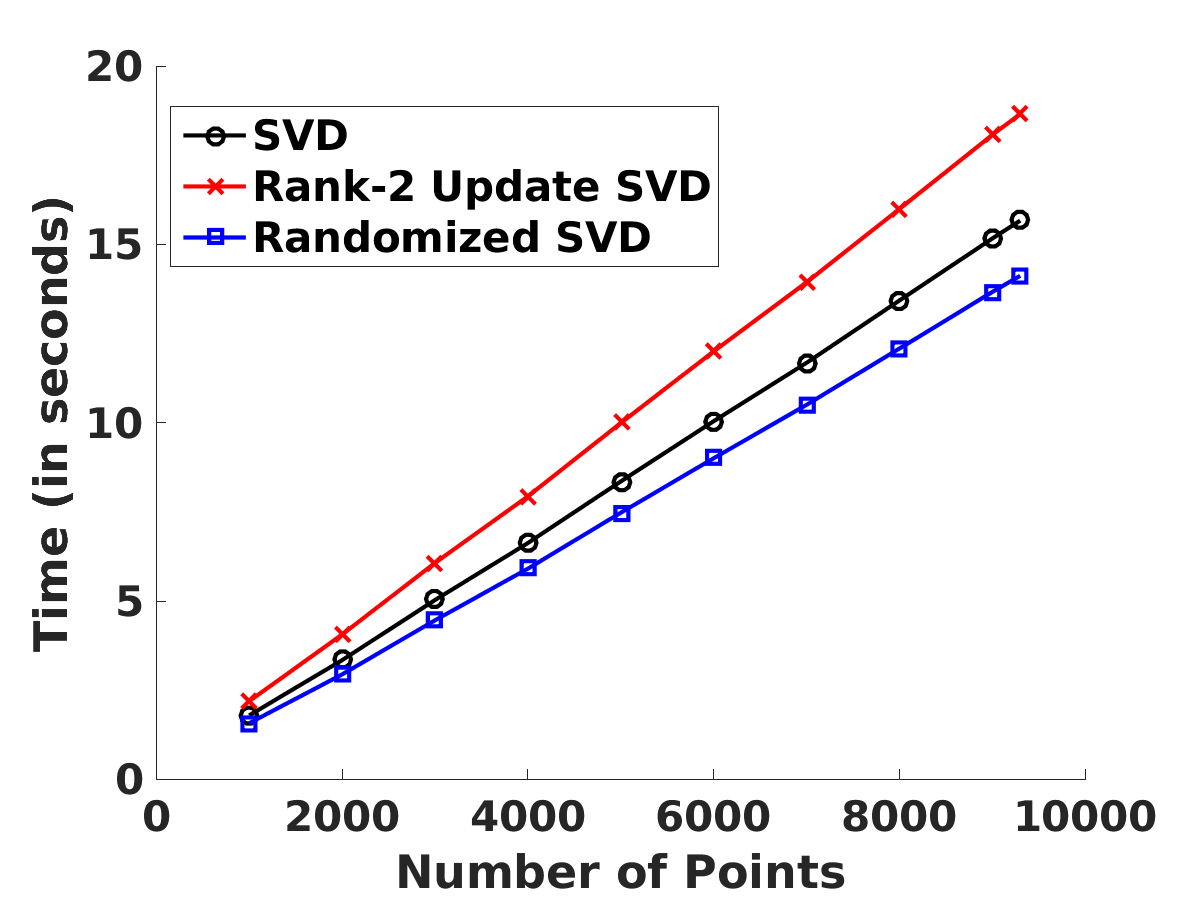}}
\caption{Time (including updating model and making predictions, not just the time for using different SVD solvers) for USPS when increasing the number of samples. Here, x axis varies the number of data points observed. As the size of the matrix for SVD update is small ( a few hundreds), and the number of samples observed is large, so the time cost for the learning process is close to linear in the number of data points observed. }
\label{fig:svdtime}
\end{figure}

In all, the OANA algorithm is given in Algorithm \ref{alg:main1}. Note that in the entire online learning process, in order to update the centroids and construct \nys mappings, we need to store the landmark points set $M$, $U$ and $S$ ($E$'s eigenvectors and eigenvalues). Therefore for a rank-$r$ mapping, the budget size is $O(md+mr)$, where $d$ is the dimension of the sample, $m$ is the number of landmark points, $r$ is the dimension of the mapping. Also $\epsilon$ in Algorithm \ref{alg:main1} controls the number of updates for landmark points and the nonlinear mapping.
\begin{algorithm}[t]
\caption{Online Adaptive \Nys Approximation}
 \label{alg:main1}
 \KwIn{$\bx_1,\cdots,\bx_T$ arrives in sequence; number of landmark points $m$, rank $r$, threshold $\epsilon$.}
 \KwOut{Nonlinear kernel mapping for each sample.}
Initialize cluster centroids $M$ as $\bu_1,\bu_2,\cdots,\bu_m\in \R^d$ randomly\;
Compute the rank-$r$ SVD of the kernel matrix $E\approx U_rS_rU_r^T$ where $E_{ij} = k(\bu_i,\bu_j)$\;
\For{$t=1,\cdots, T$}{
 Find the closest cluster $\bx_t$ belongs to, that is,
 $q = \argmin_z \|\bx_t-\bu_z\|_2^2$\;
 \uIf{$\|\bx_t-\bc_q\|^2\geq\epsilon$}{
 Update the $q$-th cluster's centroid as
 $\bu_q \leftarrow \frac{N_q\bu_q+\bx_t}{N_q+1}$\;
 Update the number of points belonging to $q$-th cluster $N_q \leftarrow N_q+1$\;
 Warm start the randomized SVD to compute the new landmark points' kernel matrix $\bar{E}$'s rank-$r$ SVD: $\bar{E}\approx U_rS_rU_r^T$ as shown in Sec.\ref{sec:oana}\;   
  }
 The approximate kernel mapping for $\bx_t$, $\varphi_M(\bx_t)$ as in Eq~(\ref{eq:mapping2})\;
 }
\end{algorithm}

\subsection{Nonlinear Online Learning with Adaptive \Nys Approximation}
\label{sec:main}
After we update the landmark points set $M$ and have a new kernel mapping $\varphi_M(\cdot)$, we need to update the model $\bw_t$ correspondingly so that the data points are trained with the new mapping. Suppose at $t$-th iteration, given $\bx_t$, we change the landmark point, and there will be three changes: landmark points set $\bar{M}$, non-linear mapping $\varphi_{\bar{M}}(\cdot)$, and the model $\bw$. To update the model $\bw$, we propose to first fit the new point $\bx_t$, and then update the model based on the new non-linear mapping. As a consequence, we follow the two-stage scheme to update $\bw$: first updating the model $\bw$ in observation of new data point $\bx_t$, and then updating the model based on the change of landmark points ($M\rightarrow\bar{M}$). 

For the first stage, we update the model by minimizing the loss for point $\bx_t$ given the landmark point set $M$:
\begin{equation}
\label{eq:linearsystem}
\min_{\bw}\frac{\lambda}{2}\|\bw\|_2^2+\ell(y_t,\bw\varphi^T_M(\bx_t)).
\end{equation}
Eq(\ref{eq:linearsystem}) is solved using gradient descent.


For the second stage, after we update the model based on $\bx_t$ and update the landmark points and mapping, we solve the following the optimization problem:
\begin{align}
\label{eq:linearsystem1}
\min_{\bar{\bw}}\sum_{i=1}^m (\bw_t\varphi^T_M(\bu_i)-\bar{\bw}\varphi^T_{\bar{M}}(\bu_i))^2+\theta\|\bar{\bw}\|^2 \\ \nonumber
\end{align}
where $\bw_t$ is the model after the stage one in Eq(\ref{eq:linearsystem}); $\bar{M} = \{\bu_1, \cdots, \bu_m\}$ is the new set of landmark points at $t$-th iteration after landmark points update; $M$ is the set of landmarks before update; $\theta$ is a regularization term. The first term in Eq.\eqref{eq:linearsystem1} measures the discrepancy in model prediction before and after changing landmark points and the corresponding feature map. The intuition is since the model $\bw_t$ works well for the old mapping, we encourage the prediction value for old and new mapping to be close. $\bar{\bw}$ is for smooth purpose so that to limit the change of the model. Eq.\eqref{eq:linearsystem1} is the standard least squares problem, which can be solved efficiently as described in \cite{Boyd:2004}.

\begin{figure}[htp]
\centering
	{\includegraphics[width=.4\textwidth]{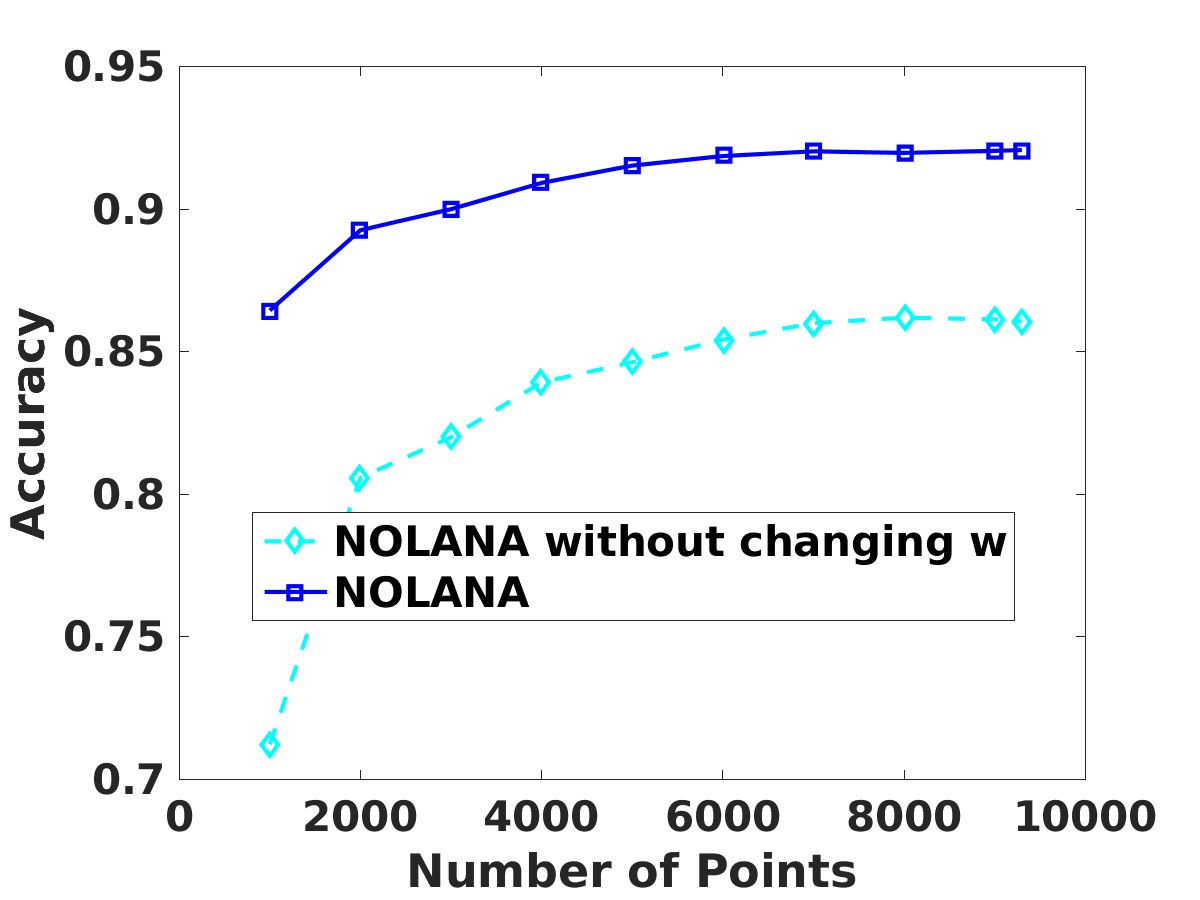}}
\caption{Effect of changing or not changing model after updating mapping in USPS dataset.}
\label{fig:changew}
\end{figure}

We show the effect of changing $\bw$ using two-stage procedure in Figure \ref{fig:changew}. The dot line in Figure \ref{fig:changew} only updates the model based on the new data point without considering the discrepancy in model prediction before and after changing the feature map, that is, does not consider the second stage Eq(\ref{eq:linearsystem1}). We can clearly see the improvement by updating $\bw$ using Eq.(\ref{eq:linearsystem1}) and necessary of taking two-stage procedure to update the model $\bw$.
\begin{algorithm}[t]
\caption{Nonlinear Online Learning with Adaptive \Nys Approximation(NOLANA)}
 \label{alg:main_2}
 \KwIn{$\bx_1,\cdots,\bx_T$ arrives in sequence; number of landmark points $m$, rank $r$}
 \KwOut{Prediction for each new coming data point}
Initialize: cluster centroids $M$ as $\bu_1,\bu_2,\cdots,\bu_m\in \R^d$ randomly; initial weight $\bw=0$\;
Construct the kernel matrix $E$ from set $M$\;
Perform eigendecomposition over $E$ such that $E \approx U_rS_rU_r^T$\;
\For{$t=1,\cdots, T$}{
 Compute the approximate kernel mapping for $\bx_t$ as $\varphi_{M}(\bx_t)$ as shown in Eq~(\ref{eq:mapping2})\;
Use $\bw_t$ to predict on new sample $\bx_t$ using $f(\bx_t) = \bw_t\varphi_M(\bx_t)$\;
 \uIf{Kernel mapping changed as shown in Algorithm \ref{alg:main1}}{
Compute the non-linear mapping for $\bu_1,\cdots,\bu_m$ under the old and new mapping, that is to get $\varphi_M(\bu_i)$ and $\varphi_{\bar{M}}(\bu_i)$\;
Update $\bw_t$ by solving Eq~(\ref{eq:linearsystem}) and Eq~(\ref{eq:linearsystem1})\;
 }
 \Else{
Receive $y_t$ and suffer loss $\ell(y_t, \bw_t\varphi^T_M(\bx_t))$\;
Update $\bw_{t+1} = \bw_t-\eta\nabla \ell(y_t, \bw_t\varphi^T_M(\bx_t))$\;
 }
 }
\end{algorithm}

By combining OANA to construct features following the above two-stage procedure to update model under the new mapping, our final algorithm Nonlinear Online Learning with Adaptive \Nys Approximation (NOLANA) is given in Algorithm~\ref{alg:main_2}. We can use NOLANA for different machine learning tasks including classification, e.g., using logistic loss and hinge loss, and regression, e.g., squared loss.

\subsection{Analysis}

{\bf Time and budget analysis:} In the entire online process as shown in Algorithm \ref{alg:main_2}, we need to store $m$ landmark points ($M$), and the eigenvectors and eigenvalues of the resulting kernel matrix ($E$). Therefore the budget size is $O(md+mr)$, which is the same as of \cite{nogd}. If we change the landmark points, we need to first update the SVD of the new set of landmark points using randomized SVD which takes $O(m^2r)$ time and update the model $\bw$, which takes $O(mr^2)$ time as well. If we do not need to change landmark point, then we just need a SGD step to update the model. As for the prediction, the time complexity is $O(md+mr)$, which is the same with the traditional \nys method's prediction time, therefore we do not cause extra time overhead for prediction.
\begin{table}
  \centering
  \caption{The accuracy, number of updates, and updating time varying $\epsilon$ in USPS dataset.}
  \label{tab:eps}
\resizebox{8.5cm}{!}{
  \begin{tabular}{|c|r|r|r|r|r|r|}
    \hline\hline
    $\epsilon$& 0&25&50&100&200&300\\
   \hline
    Accuracy($\%$)&92.00&    91.91	&92.14	&	90.91&90.27&90.52 \\
    \hline
    Time(in secs)& 30.43  &  25.98	&20.57	&	9.67&5.91&5.75\\
	\hline
    no. of updates &9298&7742&5523&1262&23&0\\
    \hline \hline
  \end{tabular}
  }
\end{table}

Note that we can control the total number of updates with $\epsilon$ (step 5 in Algorithm \ref{alg:main1}). If $\epsilon$ is small, Algorithm \ref{alg:main_2} has more updates than the case when $\epsilon$ is large. Because of that, we will have discontinuous landmark points updates. Such an update has two benefits: reducing the update time, and improving the prediction accuracy. If there are too many updates (small $\epsilon$), it takes longer updating time, and might have even higher prediction error as it keeps changing the model to fit one sample. On the other hand, if there are too few updates (large $\epsilon$), it is faster, but the model is not well adjusted to the data points resulting in low accuracy. Table \ref{tab:eps} shows the computation time (including the update and prediction time), number of updates, accuracy for different choices of $\epsilon$ on USPS dataset, where $\epsilon$ controls the number of updates. The prediction time is the same for all the $\epsilon$. When $\epsilon$ is 0, the update is continuous, that is, we update a landmark point in the budget for each new coming point. When $\epsilon=300$, there is no updates. From Table \ref{tab:eps}, we can see that $\epsilon=50$ achieves the highest accuracy, and when we increase $\epsilon$, the update time and accuracy decrease. Therefore, controlling the number of updates is important for accuracy and computational purposes, and thus there is a trade-off between accuracy and computation time when varying the distance threshold $\epsilon$.

There are several heuristics one can employ in Algorithm \ref{alg:main_2} for updating landmark points under our framework. For example, using weighted kmeans to incorporate the density information of each cluster or assigning expiration time to each landmark point to prevent the landmark points staying in the budget for too long. Extensions based on these heuristics is a future work.

{\bf Regret Analysis:} Next we show the regret analysis for our method.
\begin{theorem}
Assume we learn with kernel function $k(\cdot,\cdot)$, and some convex loss function $\ell(y,f(\bx))$ and the norm of its gradient in all iterations is bounded by a constant $L$. Given a sequence of $T$ samples $\bx_i,\cdots,\bx_T$ that form a kernel matrix $K\in R^{T\times T}$, let $\bar{K}_t$ be the \nys approximation of $K$ using our method at step $t$. In addition, define $\bw_t^*$ be the optimal solution when using \nys kernel approximation at step $t$ assuming we have access to all the instances as a batch to train a model. Let $\bw_t$ be the model our method learned.  We define $f^*$ as the optimal classifier in the original kernel space (using $K$) with observing all the instance. Then,
\begin{align}
\sum_{t=1}^{T} L_t(\bw_t)- \sum_{t=1}^{T} L_t(f^*)\leq &\frac{\|\bw_T^*\|^2}{2\eta}+\frac{\eta}{2}L^2T+\tau (R+C)  \nonumber  \\
&+\frac{1}{2T\lambda}\sum_{t=1}^{T}\|K-\bar{K}_t\|_2, \label{thm:th2}
\end{align}
where $R=\max_{t=2}^T\|\bar{\bw}_t-\bw_t\|^2$, the maximum change we made when changing the mapping at step $t$; $C=\max_{t=2}^T\|{\bw}_{t+1}^*-\bw_t^*\|^2$, the maximum difference between optimal model under old and new mappings; and $\tau$ is the number of times we adjust mapping; $\eta$ is the learning rate; $\lambda$ is the regularization parameter.
\end{theorem}


\begin{proof}
Let $\bar{w}_t$ be the model before we adjust the model after changing the \nys feature map. As we use SGD to update the model after seeing $\bx_t$ and $y_t$, we have the following:
\begin{align}
&\|\bar{\bw}_{t+1}-\bw_t^{*}\|^2\\
&=\|\bw_t-\eta \nabla\ell_t(\bw_t)-\bw_t^*\|^2 \nonumber\\
&=\|\bw_t-\bw_t^*\|^2+\eta^2\|\nabla\ell_t(\bw_t)\|-2\eta\nabla\ell_t(\bw_t)(\bw_t-\bw_t^*). \nonumber
\end{align}
Due to the convexity of the loss function, we have
\begin{equation}
\ell_t(\bw_t)-\ell_t(\bw_t^*)\leq \nabla\ell_t(\bw_t)(\bw_t-\bw_t^*).
\end{equation}
Hence,
\begin{equation}
\ell_t(\bw_t)-\ell_t(\bw_t^*)\leq \frac{\|\bw_t-\bw_t^*\|^2-\|\bar{\bw}_{t+1}-\bw_t^*\|^2}{2\eta}+\frac{\eta}{2}\|\nabla\ell_t(\bw_t)\|^2.
\end{equation}
Summing the above over $t=1,\cdots,T$ leads to
\begin{align}
\label{eq:inequality}
&\sum_{t=1}^{T}(\ell_t(\bw_t)-\ell_t(\bw_t^*))\\ \nonumber
&\leq \frac{\|\bw_1-\bw_1^*\|^2-\|\bar{\bw}_{2}-\bw_1^*\|^2+\cdots+\|\bw_{T}-\bw_T^*\|^2-\|\bar{\bw}_{T+1}-\bw_{T}^*\|^2}{2\eta}\\ \nonumber
&\hspace{15pt} +\frac{\eta}{2}\sum_{t=1}^{T}\|\nabla\ell_t(\bw_t)\|^2\nonumber\\ 
&\leq \frac{\|\bw_1-\bw_1^*\|^2-\|\bar{\bw}_{T+1}-\bw_{T}^*\|^2}{2\eta}+\frac{\sum_{t=2}^{T}\|\bar{\bw}_t-\bw_t\|^2}{2\eta}\nonumber\\
&\hspace{15pt}+\frac{\sum_{t=1}^{T}\|\bw_t^*-\bw_{t+1}^*\|^2}{2\eta} +\frac{\eta}{2}\sum_{t=1}^{T}\|\nabla\ell_t(\bw_t)\|^2 \nonumber\\ 
&\leq \frac{\|\bw_1-\bw_1^*\|^2-\|\bar{\bw}_{T+1}-\bw_T^*\|^2}{2\eta}+\tau (R+C)+\frac{\eta}{2}L^2T \nonumber\\ 
&\leq \frac{\|\bw_T\|^2}{2\eta}+\tau (R+C)+\frac{\eta}{2}L^2T. \nonumber
 \end{align}

As proved in \cite{nogd}, the linear optimization problem in kernel space is equivalent to the approximate problem in the functional space of \nys method. Therefore, based on Lemma 1 in \cite{nogd} and Eq(\ref{eq:inequality}), we have
\begin{align}
\sum_{t=1}^{T} L_t(\bw_t)- \sum_{t=1}^{T} L_t(f^*)&\leq \frac{\|\bw^*_T\|^2}{2\eta}+\frac{\eta}{2}L^2T+\tau (R+C)\\ \nonumber
&\hspace{15pt} +\frac{1}{2T\lambda}\sum_{t=1}^{T}\|K-\bar{K}_t\|_2. 
\end{align}

Note that $C$ is a constant related to the loss function; $R$ is sum of changes in the model $\bw$ along the learning process. And $\tau$ is a parameter in our algorithm and can be controlled. 
\end{proof}
It has been shown in many works that kmeans \nys has better approximation than \nys with random sampling. Also as illustrated in Figure \ref{fig:appfig} that compared to the standard \nys, our online \nys mapping achieves lower $\|K-\bar{K}\|_2$. As shown in \cite{tian12}, we have
$\|\tilde{K}-K\|_2\leq O(\frac{T}{B})$, where $\tilde{K}$ is the kernel by standard \nys, similar to the analysis in \cite{nogd} and set $\eta=O(\frac{1}{\sqrt{T}})$,$B=O(\sqrt{T})$, and $\tau=O(\sqrt{T})$, we have
\begin{equation}
\sum_{t=1}^{T} L_t(w_t)- \sum_{t=1}^{T} L_t(f^*)\leq O(\sqrt{T})
\end{equation}

\begin{figure*}[tb]
  \centering
  \begin{tabular}{ccc}
    \subfigure[cpusmall]{
    \label{fig:toyk-4}
    \includegraphics[width=0.3\linewidth]{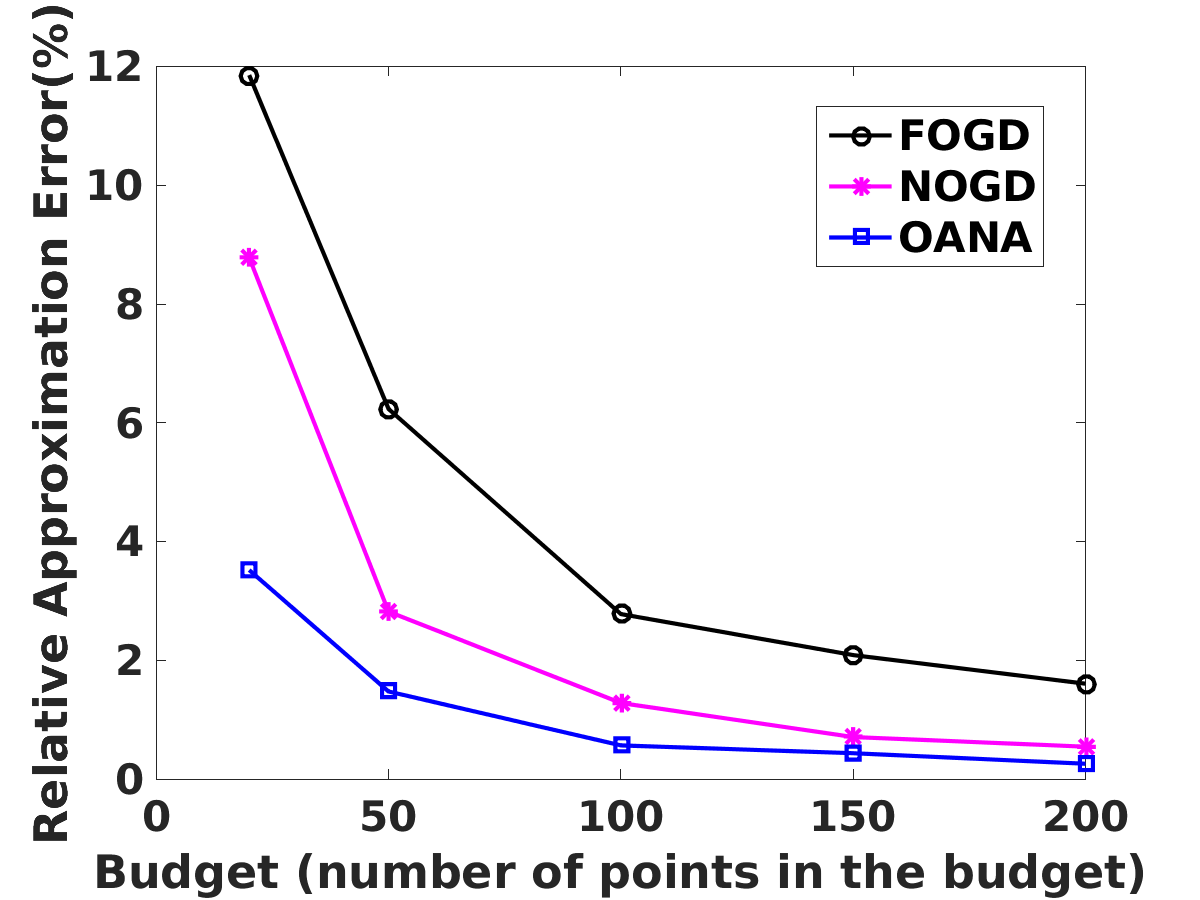}
    }&\hspace{-15pt} 
    \subfigure[usps]{
    \label{fig:toyk-4}
    \includegraphics[width=0.3\linewidth]{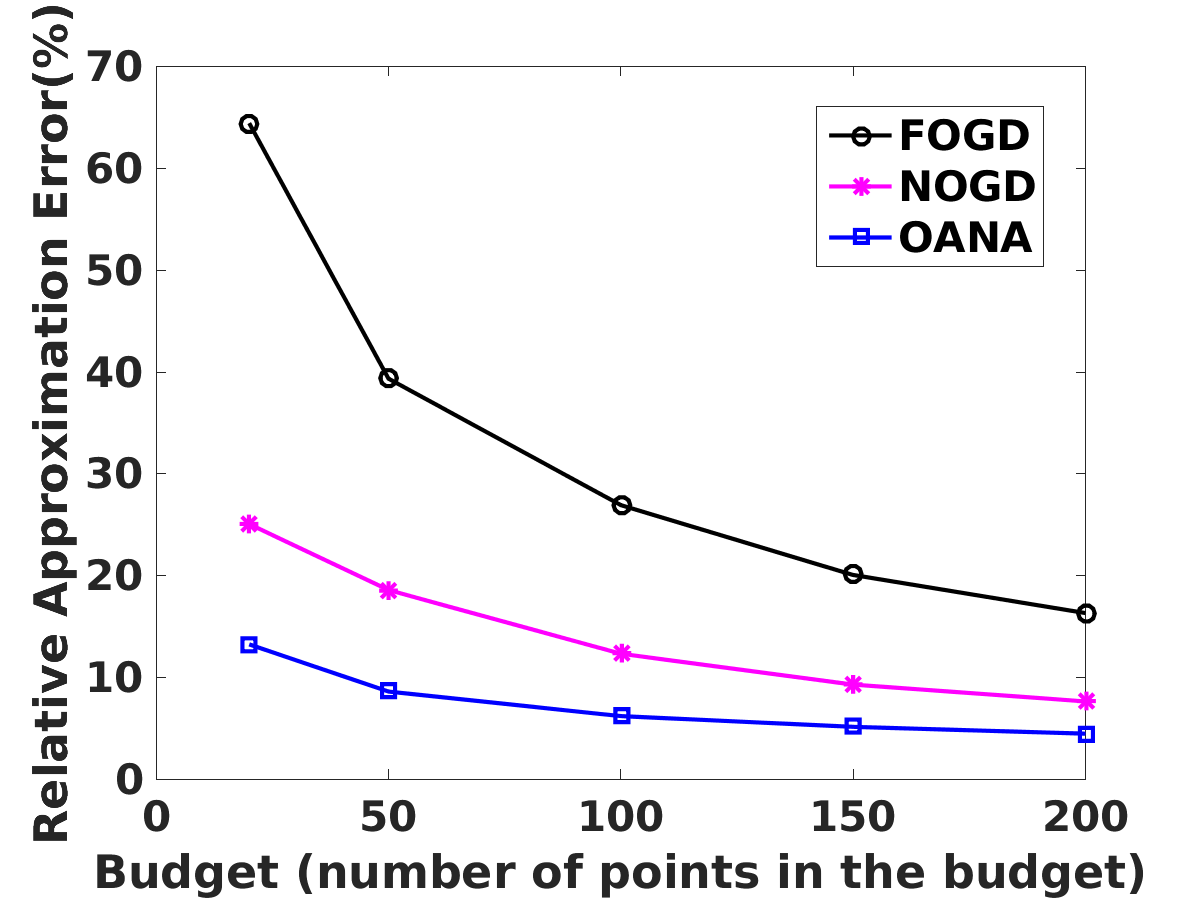}
    } &\hspace{-15pt} 
    \subfigure[covtype]{
    \label{fig:toyk-4}
    \includegraphics[width=0.3\linewidth]{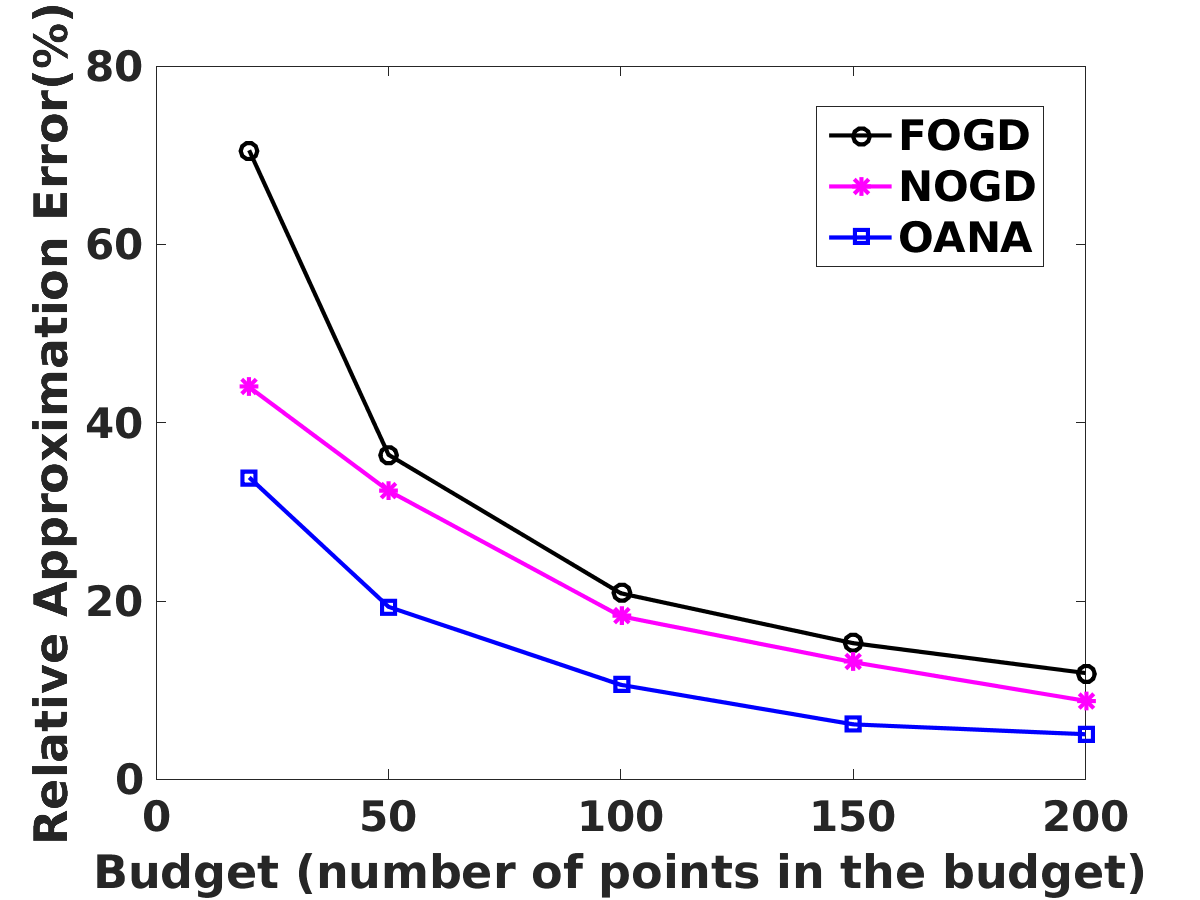}
    }
  \end{tabular}
  \caption{ The budget size (the number of landmark points $m$) versus relative kernel approximation error. For FOGD, we set the projection dimension to be $\frac{md+mr}{d}$ so that each method has the same budget size. }
  \label{fig:appfig}
\end{figure*}

Note that the importance of each term in Eq\eqref{thm:th2} depends on many factors, e.g., the kernel approximation quality, the frequency with which we update of the model $\bw$, and the update parameters. For example, if $B=\sqrt{T}$ and we choose the learning rate $\eta=\frac{1}{\sqrt{T}}$, then each component in the bound in Eq~\eqref{thm:th2} will have the same complexity, $O(\sqrt{T})$.

The difference between our regret bound in Eq(\ref{thm:th2}) and Theorem 2 in~\cite{nogd} is that we have one additional term to bound the change in model $\bw$ due to the changing landmarks. Meanwhile, as our method has lower approximation error (shown experimentally in Figure~\ref{fig:appfig} than NOGD, we have smaller error in the fourth term in Eq(\ref{thm:th2}) compared to the third term in Theorem 2 in ~\cite{nogd}. Furthermore, we have shown our method achieves higher accuracy than~\cite{nogd} in the experiments.

\section{Experiments}
\label{sec:exp}
In this section, we first show the quality of the nonlinear mapping based on the kernel approximation error. We then consider three important machine learning tasks: binary classification with hinge loss and logistic loss, and regression with square loss. In particular, we compare the following online learning methods:
\begin{enumerate}[noitemsep,nolistsep,leftmargin=*]
\item Passive-Aggressive (PA)\cite{Crammer:2006}: a linear online learning method that aims at fitting the current model to the latest example.
\item Fourier Online Gradient Descent (FOGD)\cite{nogd}: uses random Fourier features\cite{Rahimi:2007} in online gradient descent to update the model.
\item \nys Online Gradient Descent (NOGD)\cite{nogd}: uses nonlinear features from \nys method in online gradient descent to update the model.
\item Dual Space Gradient Descent (DualSGD)\cite{Le2016DualSG}: uses random Fourier feature to store the information of the removed data points to maintain the budget.
\item Reparametered Random Feature(RRF) \cite{ijcai2017}: learning the distribution parameters in random Fourier feature along with the online learning process.
\item Online Adaptive \Nys Approximation (OANA): our proposed online kernel approximation method that adaptively updates the landmark points as the online process advances.
\item Nonlinear Online Learning with Adaptive \Nys Approximation (NONALA): our proposed method that adaptively updates the landmark points and nonlinear mapping, and changes model according to the new mapping. 
\end{enumerate}

\begin{figure*}[tb]
\centering
\subfigure[cpusmall]{\includegraphics[width=.3\textwidth]{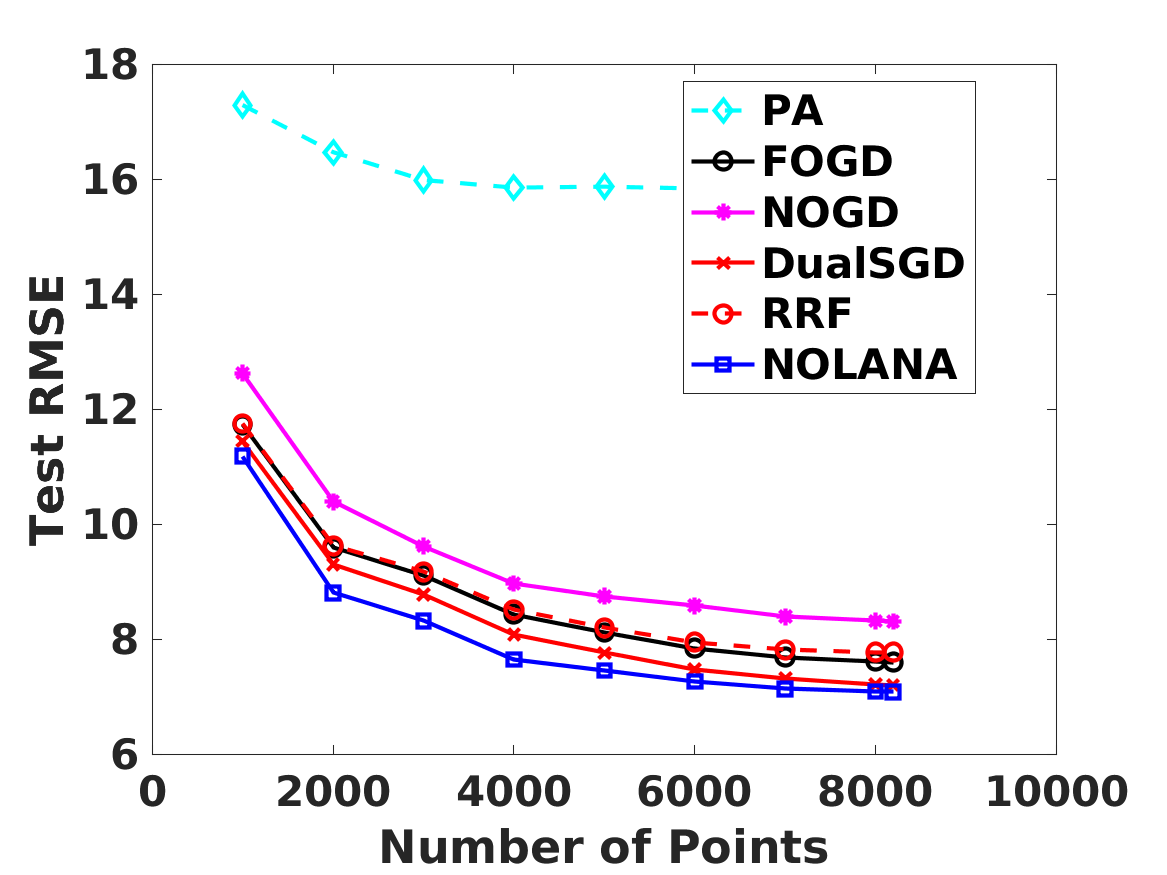}}\quad
\subfigure[usps]{\includegraphics[width=.3\textwidth]{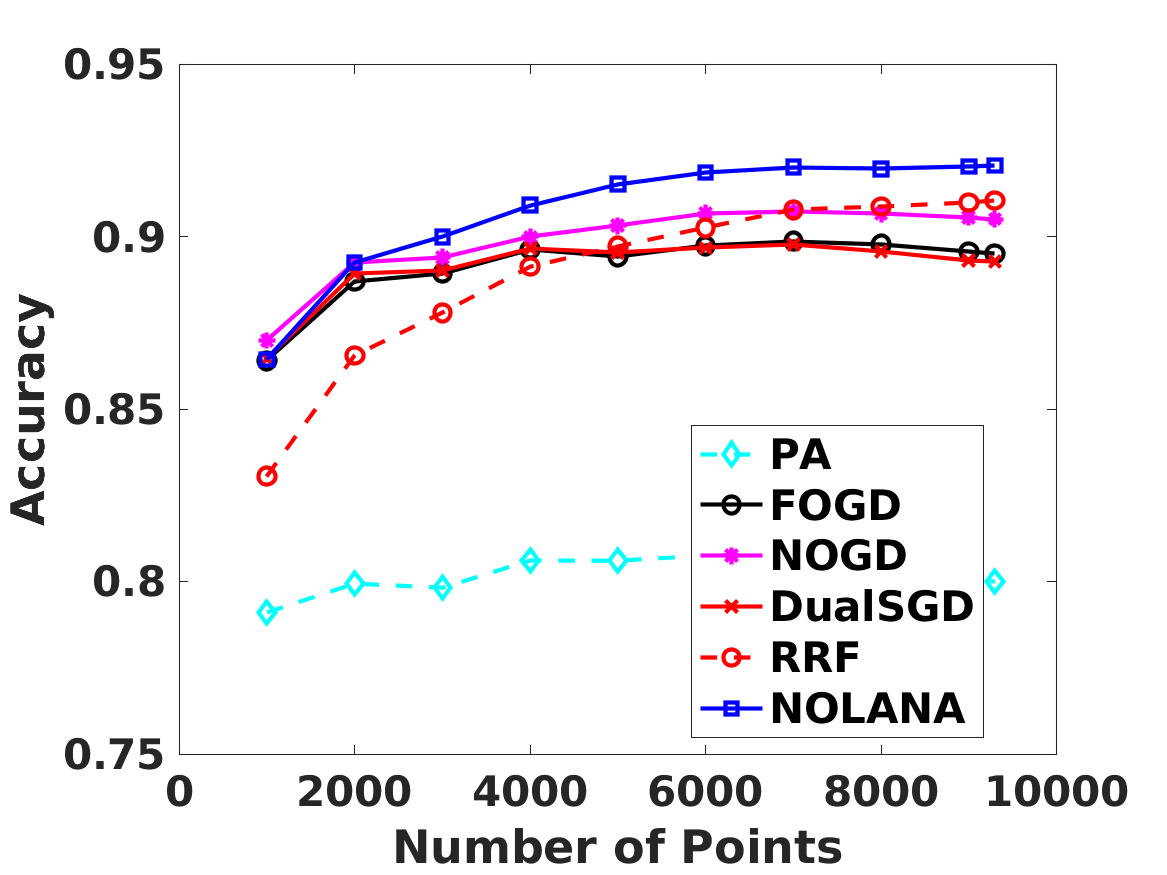}}\quad
\subfigure[ijcnn]{\includegraphics[width=.3\textwidth]{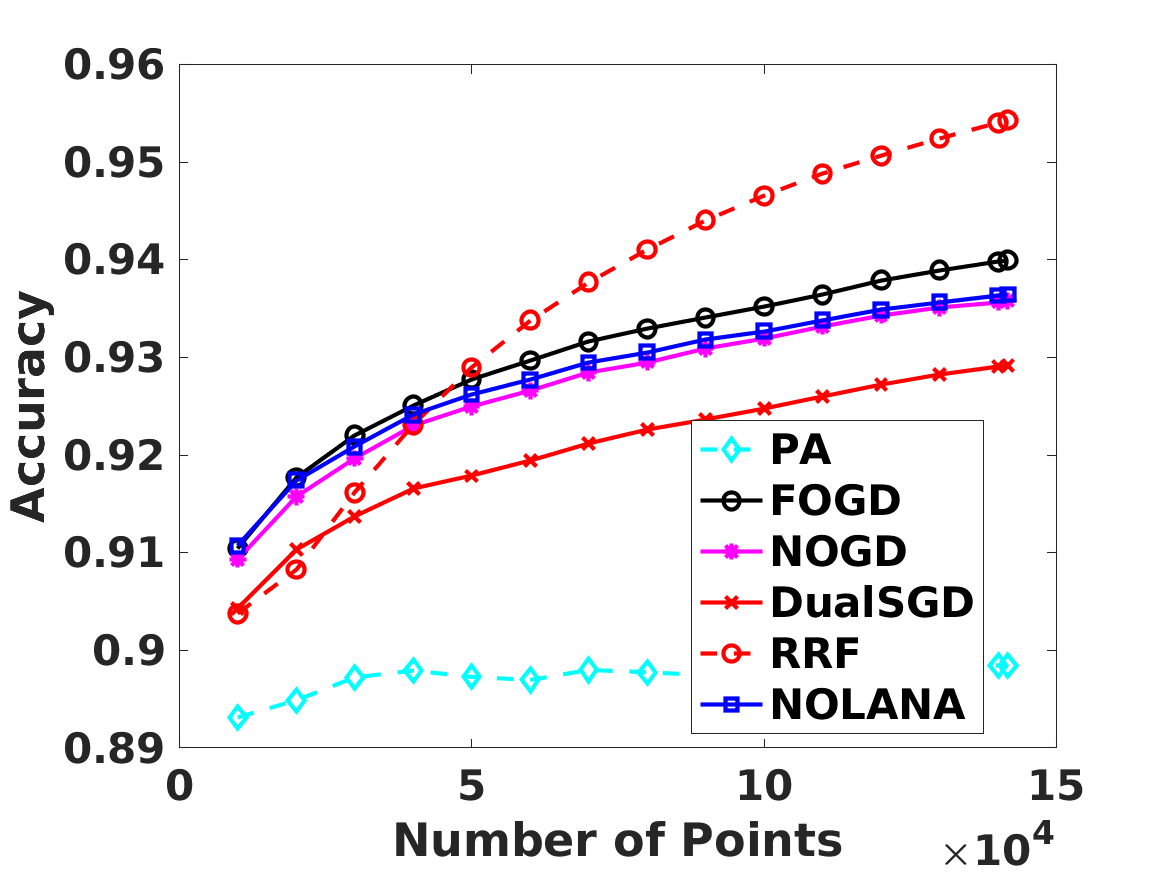}}
\medskip
\subfigure[webspam]{\includegraphics[width=.3\textwidth]{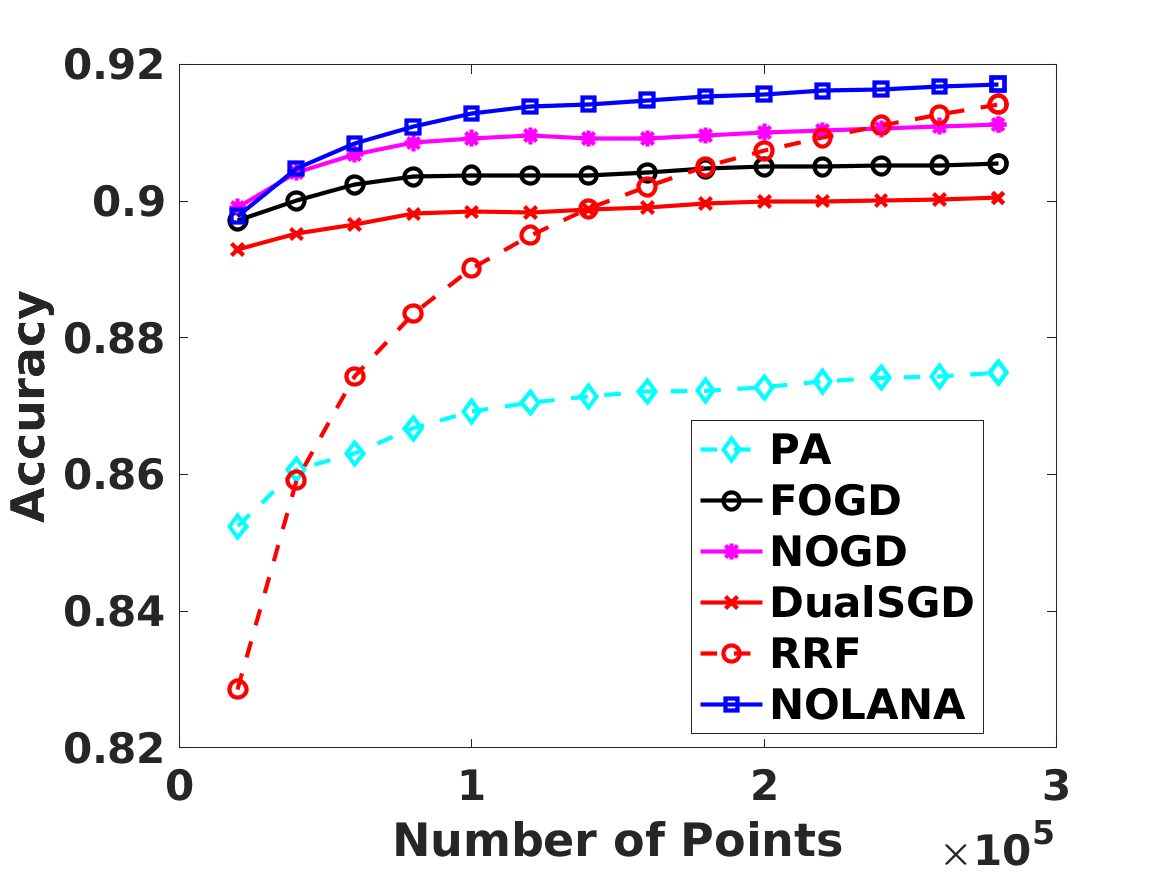}}\quad
\subfigure[covtype]{\includegraphics[width=.3\textwidth]{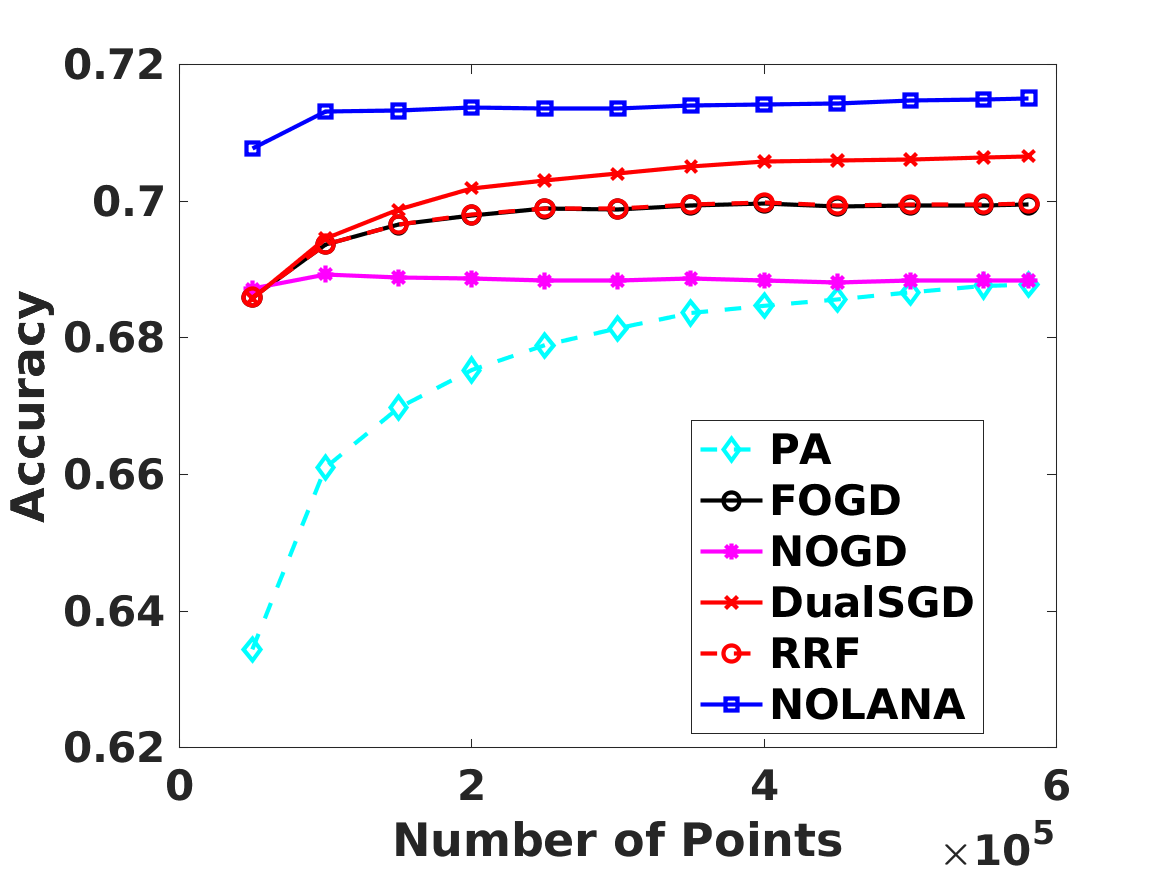}}
\vspace{-10pt}
\caption{ The number of training points vs accuracy (logistic loss and hinge loss) and Test RMSE (square loss). For FOGD, we use the projection dimension to be $\frac{md+mr}{d}$, so that each method has the same budget size. PA does not need to save any data point in the budget.}
\label{fig:mainfig}
\end{figure*}
  \vspace{10pt}
\begin{table*}[tb]
  \centering
  \caption{Accuracy comparison of our method with state-of-the-art online learning methods for online classification and regression. The results are TestRMSE result (on cpusmall dataset) and prediction accuracy (for the rest four datasets). NOLANA is our proposed method. $n$: number of data points; $d$: dimension of the data; $m$: number of data points in the budget set. Budget for all methods was kept the same (except PA) for each dataset.}
  \label{tab:method_compare}
\resizebox{18cm}{!}{
  \begin{tabular}{|c|r|r|r|r|r|r|r|r|r|r|r|r|}
    \hline\hline
    Dataset&$n$&$d$&$m$&Loss& task&PA &NOGD & FOGD& DualSGD& RRF&NOLANA \\
   \hline
    cpusmall&8,192&12&20&square& regression& 15.83 & 8.30 & 7.60 &7.20& 7.77 &{\bf 7.08} \\
    \hline
    usps&9,298&256&100&hinge&classification&80.01& 90.52 & 89.52  &89.30&91.05 &{\bf 92.07} \\
\hline
 	ijcnn &141,691&22&100&hinge&classification&89.84 &93.58&94.00 &92.92& {\bf 95.43} &93.65\\
    \hline
	webspam&280,000&254&100&hinge&classification&87.49 &91.13 & 90.55  &90.05 &91.42& {\bf 91.71} \\
	\hline
	covtype&581,012&54&200&logistic&classification&68.78 &68.84 & 69.95  &70.66 &69.97& {\bf 71.51} \\
	\hline\hline    
  \end{tabular}
  }
\end{table*}
For all the above methods, except for PA, which is a linear method, we use Gaussian kernel as the base kernel. All the experiments are conducted on a Ubuntu PC machine with an Intel Xeon X5440 2.83GHz CPU and 32G RAM with datasets from the UCI data repository and their statistics is shown in Table \ref{tab:method_compare}.

For each dataset, we keep the size of budget to be same for all the nonlinear methods, except for PA as PA does not need budget for data points. For NOGD and our method, we use the same $m$, i.e., the number of landmark points, and the budget size is $md+mr$ where $r$ is the reduced dimension ($r=0.8m$). For FOGD, we need to store the projection direction to construct random features. To ensure that FOGD uses the same budget as others, we set the projection dimension to be $\frac{md+mr}{d}$.  For DualSGD, which needs to store both projections and support vectors, we store $m$ support vectors, and the dimension of the projection is $\frac{mr}{d}$. Because \cite{Le2016DualSG,nogd} provides comprehensive comparison among DualSGD, NOGD, FOGD with other kernel based online learning methods, e.g., BPA~\cite{wang10b} and Projectron~\cite{Orabona:2008}, we ignore comparison with them here. Also we tested doubly Stochastic Functional Gradients (Doubly SGD)\citep{NIPS2014_5238}. The performance is worse than the methods compared in this paper. For instance, Doubly SGD achieves 75\% accuracy in usps dataset, while all the other methods achieves accuracy higher than 80\%. The reason is that Doubly SGD only updates one coordinate at a time for each data point, and thus needs several iterations over the entire dataset or a large amount of data points to achieve good performance. For online learning setting, we consider the data streaming and iterate the data just once.

\subsection{Kernel Approximation}
The kernel approximation results for varying budget sizes on three datasets are given in Figure~\ref{fig:appfig} to show how good the nonlinear mapping is. The three datasets' statistics are shown in Table~\ref{tab:method_compare}. For covtype, we randomly sampled 10,000 data points for kernel approximation experiments, and for the other two datasets we use all the data points. We use the relative kernel approximation error $\frac{\|G-\bar{G}\|_F}{\|G\|_F}$ to measure the quality. For our method, OANA, $\bar{G}$ is constructed using the landmark points after going through all the data points. For NOGD, $\bar{G}$ is constructed based on the first $m$ points in the data stream, and for FOGD, the dimension of the projected space is $\frac{md+mr}{d}$, so that all the three methods have the same budget size. As shown in the Figure~\ref{fig:appfig}, we can see that our method achieves lower approximation error than fixing the first $m$ points as landmark points (NOGD) and using random features as feature mapping (FOGD). Therefore, our method can generate better feature mapping, and benefit online learning tasks, such as classification and regression.

\subsection{Classification and Regression}
We compare our method with state-of-the-art online learning methods for solving online classification and regression problems. We test these methods on five datasets and three loss functions for classification and regression. For each dataset, we randomly shuffle the dataset for 5 times, and the result is the average of 5 independent runs. The results are shown in Table~\ref{tab:method_compare} and Figure~\ref{fig:mainfig}. For each dataset, we fix the budget to be the same for all the methods (except PA). The parameters are chosen by 5-fold cross-validation. Table~\ref{tab:method_compare} reports the accuracy for classification task (hinge loss and logistic loss) and testRMSE for regression task (square loss) after we process all the data points. We can see that in 4 out of 5 datasets, our method performs the best, showing adaptively learning landmark points is beneficial for online learning with \nys mapping. In Figure~\ref{fig:mainfig}, we plot the prediction accuracy or testRMSE for all the methods in response to the number of points observed so far. We can see that the models becomes better when training with more examples, with our method performing better than others in most cases.

We also want to mention that all the methods in Table~\ref{tab:method_compare} are online learning algorithms dealing with streaming data where we could go through the data just once. On the other hand, in the offline learning setting, the model is trained by going through data multiple times leading to higher accuracy than their online counterparts. Therefore our method and all the other online methods shown in the paper has lower accuracy than the commonly known offline results~\cite{HsuLibsvmTutorial2003} on the same datasets from the literature.

About the running time, to predict new data point, our method has the same prediction time with traditional \nys (NOGD) (as analyzed in Section 4.2). Therefore we do not introduce extra time overhead for prediction. For updating the model part, we need some extra time for updating the landmark points on-the-fly to achieve higher accuracy. Although both have the same time complexity for update, the total update time is more than NOGD in practice. We have shown the time comparison in Table~\ref{tab:eps} that there is a trade-off between the computation time and accuracy. More updates normally cost more time, but the model will have higher accuracy. On the other hand, when $\epsilon$ in Algorithm \ref{alg:main1} is 0, there is no landmark updates, and our method is the same as the \nys method. For example, the computation time (including time for prediction and updating the model) of our method is 5.75 sec with accuracy of 90.52\%. When $\epsilon$ is 50, our method takes more updates, and accuracy goes to 92.14\% taking longer time 20.57 sec. 

\section{Conclusion}
\label{sec:conclusion}
In this paper our goal was to improve \nys method and apply it for online learning tasks. The challenge is that we cannot apply state-of-the-art \nys methods under online learning setting, as they need to access the entire data to compute good landmark points for approximation. We propose to adaptively adjust the landmark points via online kmeans, followed by a two-stage optimization scheme to update the model through the online learning process. Extensive experiments indicate that our method performs better than state-of-the-art online learning algorithms for online classification and regression.
\bibliographystyle{ACM-Reference-Format}
\bibliography{sample}
\end{document}